\documentclass[twoside,11pt]{article}

\usepackage{enumerate}
\usepackage{amsmath}
\usepackage{amsfonts}
\usepackage{amssymb}
\usepackage{amsbsy}
\usepackage{isomath}
\usepackage{dsfont}
\usepackage{algorithm}
\usepackage{algorithmic}
\usepackage{mathrsfs}
\usepackage{paralist}
\usepackage{epsfig}
\usepackage{subfigure}
\usepackage{makeidx} 
\usepackage{color}
\usepackage{tikz}
\usepackage{textcomp}
\usepackage{preview}
\usepackage{float}
\usepackage{pgfplots}
\usetikzlibrary{pgfplots.groupplots}
\usetikzlibrary{positioning}

\usepgfplotslibrary{patchplots,colormaps}

\numberwithin{equation}{section}

\newcommand \E {\mathop{\mbox{\ensuremath{\mathbb{E}}}}\nolimits}

\renewcommand \Pr {\mathop{\mbox{\ensuremath{\mathbb{P}}}}\nolimits}

\newcommand{\cset}[2]{\left\{\, #1 ~\middle|~ #2 \,\right\} }

\newcommand\Reals {{\mathds{R}}}

\newcommand \CA {{\mathcal{A}}}

\newcommand \CM {{\mathcal{M}}}

\newcommand \CS {{\mathcal{S}}}
\newcommand \CT {{\mathcal{T}}}

\newcommand \CZ {{\mathcal{Z}}}

\newcommand \defn {\mathrel{\triangleq}}

\newcommand \argmax{\mathop{\rm arg\,max}}

\newcommand \trace{\mathop{\rm tr}}

\newcommand \norm[1]{\left\|#1\right\|}

\DeclareMathAlphabet{\mathpzc}{OT1}{pzc}{m}{it}

\newcommand \Normal {\mathop{\mathpzc{N}}\nolimits}
\newcommand \Wishart {\mathop{\mathpzc{W}}\nolimits}

\newcommand \Student {\mathop{\mathpzc{Student}}\nolimits}

\newcommand \Uniform{\mathop{\mathpzc{Unif}}\nolimits}

\newcommand \pol {\pi}
\newcommand \Pols {\Pi}
\newcommand \mdp {\mu}
\newcommand \MDPs {\CM}

\newcommand\ind[1]{\mathop{\mbox{\ensuremath{\mathbb{I}}}}\left\{#1\right\}}

\newcommand\dd{\,\mathrm{d}}

\newcommand \MA {\matrixsym{A}}

\newcommand \MV {\matrixsym{V}}
\newcommand \MW {\matrixsym{W}}

\newcommand \vs {\vectorsym{s}}
\newcommand \vx {\vectorsym{x}}

\newcommand \vz {\vectorsym{z}}

\newcommand \vt {\vectorsym{\theta}}

\newcommand \pn[1] {\vz_{[#1]}}

\newcommand \ctx {f}

\newcommand \bel {p}
\newcommand \mbel {\Normal}
\newcommand \pbel {\Wishart}

\newcommand \pmean {\matrixsym{M}}
\newcommand \pcov {\matrixsym{C}}
\newcommand \pwish {\matrixsym{W}}
\newcommand \porder {n}

\newcommand \trans[1] {#1^\top}

\newcommand \noise {\vectorsym{\varepsilon}}

\newcommand \node[1] {\nu_{#1}}
\newcommand \pt {p_t}

\newcommand \CSet {G}
\newcommand \RSet {\hat{G}}

\newcommand \Children[1] {\mathfrak{C}(#1)}

\newcommand \Descendants[1] {\mathfrak{D}_t(#1)}
\newcommand \metric[2] {\psi(#1, #2)}
\newcommand \zooming {\zeta}
\newcommand \depth[1] {d(#1)}
\newcommand \level[1] {\ell(#1)}
\newcommand \kLSTD {K_{\textrm{L}}}

\newcommand \kAPI {K_{A}}
\newcommand \dimS {m}
\newcommand \dimLSTD {m_L}
\newcommand \nsam {n_{s}}

\newcommand \disc {\gamma}

\def\clap#1{\hbox to 0pt{\hss#1\hss}}

\def\mathrlap{\mathpalette\mathrlapinternal}
\def\mathclap{\mathpalette\mathclapinternal}

\def\mathrlapinternal#1#2{%
           \rlap{$\mathsurround=0pt#1{#2}$}}
\def\mathclapinternal#1#2{%
           \clap{$\mathsurround=0pt#1{#2}$}}

\usetikzlibrary{automata}
\usetikzlibrary{topaths}

\tikzstyle{select}=[rectangle,draw=black,fill=red!20,inner sep=0mm, minimum size=9mm]
\tikzstyle{hidden}=[circle,dashed,draw=black,fill=green!20,inner sep=0mm, minimum size=9mm]
\tikzstyle{observed}=[circle,draw=black,fill=blue!10,inner sep=0mm, minimum size=9mm]

\usepackage{pgfplots}
\usepackage{tikz}
\pgfplotsset{compat=newest} 

\usepackage{jmlr2e}

\jmlrheading{1}{2014}{}{}{}{N. Tziortziotis, C. Dimitrakakis, K. Blekas}

\ShortHeadings{Cover Tree Bayesian RL}{Tziortztiotis et al.}
\firstpageno{1}

\begin{document}

\title{Cover tree Bayesian reinforcement learning}

\author{\name Nikolaos Tziortziotis \email ntziorzi@gmail.com\\
       \addr Department of Computer Science and Engineering\\
       University of Ioannina\\
       GR-45110, Greece
       \AND
       \name Christos Dimitrakakis \email christos.dimitrakakis@gmail.com\\
       \addr Department of Computer Science and Engineering\\
       Chalmers university of technology\\
       SE-41296, Sweden
       \AND
       \name Konstantinos Blekas \email kblekas@cs.uoi.gr\\
       \addr Department of Computer Science and Engineering\\
       University of Ioannina\\
       GR-45110, Greece
     }

\editor{-}

\maketitle
\begin{abstract}
This paper proposes an online tree-based Bayesian approach for reinforcement learning. For inference, we employ a generalised context tree model. This defines a distribution on multivariate Gaussian piecewise-linear models, which can be updated in closed form. The tree structure itself is constructed using the cover tree method, which remains efficient in high dimensional spaces. We combine the model with Thompson sampling and approximate dynamic programming to obtain effective exploration policies in unknown environments. The flexibility and computational simplicity of the model render it suitable for many reinforcement learning problems in continuous state spaces. We demonstrate this in an experimental comparison with a Gaussian process model, a linear model and simple least squares policy iteration.
\end{abstract}


\section{Introduction}
In reinforcement learning, an agent must learn how to act in an unknown environment from limited feedback and delayed reinforcement. Efficient learning and planning requires models of the environment that are not only general, but can also be updated online with low computational cost. In addition, probabilistic models allow the use of a number of near-optimal algorithms for decision making under uncertainty. While it is easy to construct such models for small, discrete environments, models for the continuous case have so far been mainly limited to parametric models, which may not have the capacity to represent the environment (such as generalised linear models) and to non-parametric models, which do not scale very well (such as Gaussian processes). 

In this paper, we propose a non-parametric family of tree models, with a data-dependent structure constructed through the cover tree algorithm, introduced by~\citet{cover-tree:icml2006}. Cover trees are data structures that cover a metric space with a sequence of data-dependent partitions. They were initially proposed for the problem of $k$-nearest neighbour search, but they are in general a good method to generate fine partitions of a state space, due to their low complexity, and can be applied to any state space, with a suitable choice of metric. In addition, it is possible to create a statistical model using the cover tree as a basis. Due to the tree structure, online inference has low (logarithmic) complexity.

In this paper, we specifically investigate the case of a Euclidean state space. For this, we propose a model generalising the context tree weighting algorithm proposed by~\citet{willems:context}, combined with Bayesian multivariate linear models. The overall prior can be interpreted as a distribution on piecewise-linear models. We then compare this model with a Gaussian process model, a single linear model, and the model-free method least-squares policy iteration in two well-known benchmark problems in combination with approximate dynamic programming and show that it consistently outperforms other approaches.

The remainder of the paper is organised as follows. 
Section~\ref{sec:setting} introduces the setting, Section~\ref{sec:related-work} discusses related work and Section~\ref{sec:our-contribution} explains our contribution. The model and algorithm are described in Section~\ref{sec:cover-tree-bayesian}. Finally, comparative experiments are presented in Section~\ref{sec:experiments} and we conclude with a discussion of the advantages of cover-tree Bayesian reinforcement learning and directions of future work in Section~\ref{sec:conclusion}.

\subsection{Setting}
\label{sec:setting}
We assume that the agent acts within a fully observable discrete-time Markov decision process (MDP), with a metric state space $\CS$, for example $\CS \subset \Reals^\dimS$. At time $t$, the agent observes the current environment state $\vs_t \in \CS$, takes an action $a_t$ from a discrete set $\CA$, and receives a reward $r_t \in \Reals$. The probability over next states is given in terms of a transition kernel $P_\mdp(S \mid \vs, a) \defn \Pr_\mdp(\vs_{t+1} \in S \mid \vs_{t} = \vs, a_t = a )$. The agent selects its actions using a \emph{policy} $\pol \in \Pols$, which in general defines a conditional distribution $\Pr^\pol(a_t \mid \vs_1, \ldots, \vs_t, a_1, \ldots, a_{t-1}, r_1, \ldots, r_{t-1})$ over the actions, given the history of states and actions. This reflects the learning process that the agent undergoes, when the MDP $\mdp$ is unknown.

The agent's \emph{utility} is $U \defn \sum_{t=0}^\infty \disc^t r_t$, the discounted sum of future rewards, with $\disc \in (0,1)$ a discount factor such that rewards further into the future are less important than immediate rewards. 
The goal of the agent is to maximise its expected utility:
\begin{align}
  \max_{\pol \in \Pols} \E^{\pol}_{\mdp} U
  &=
  \max_{\pol \in \Pols} \E^{\pol}_{\mdp} \sum_{t=0}^\infty \disc^t r_t,
  \label{eq:utility}
\end{align}
where the value of the expectation depends on the agent's policy $\pol$ and the environment $\mdp$. If the environment is known, well-known dynamic programming algorithms can be used to find the optimal policy in the
discrete-state case~\citep{Puterman:MDP:1994}, while many approximate algorithms exist for continuous environments~\citep{BertsekasTsitsiklis:NDP}.
In this case, optimal policies are memoryless and we let $\Pols_1$ denote the set of memoryless policies.  Then MDP and policy define a Markov chain with kernel $P_\mdp^\pol(S \mid \vs, a) = \sum_{a \in A} P_\mdp(S \mid \vs, a) \pol(a \mid \vs)$.

However, since the environment $\mdp$ is unknown, the above maximisation is ill-posed. In the Bayesian framework for reinforcement learning, this problem is alleviated by performing the maximisation conditioned on the agent's belief about the true environment $\mdp$. This converts the problem of reinforcement learning into a concrete, optimisation problem. However, this is generally extremely complex, as we must optimise over all history-dependent policies.

More specifically, the main assumption in Bayesian reinforcement learning is that the environment $\mdp$ lies in a given set of environments $\MDPs$. In addition, the agent must select a subjective prior distribution $\bel(\mdp)$ which encodes its belief about which environments are most likely. The Bayes-optimal expected utility for $\bel$ is:
\begin{align}
  U^*_\bel \defn \max_{\pol \in \Pols^D} \E_{\bel}^{\pol} U
  = 
  \max_{\pol \in \Pols^D} \int_{\MDPs} \left(\E_{\mdp}^{\pol} U\right) \dd{\bel}(\mdp).
  \label{eq:bayes-utility}
\end{align}
Unlike the known $\mdp$ case, the optimal policy may not be memoryless, as our belief changes over time. This makes the optimisation over the policies significantly harder~\citep{duff2002olc}, as we have to consider the set of all history-dependent deterministic policies, which we denote by $\Pols^D \subset \Pols$. In this paper, we employ the simple, but effective, heuristic of Thompson sampling~\citep{thompson1933lou,wyatt1998exploration,dearden98bayesian,strens2000bayesian} for finding policies. This strategy is known by various other names, such as probability matching, stochastic dominance, sampling-greedy and posterior sampling. Very recently~\citet{osband:thompson:nips:2013} showed that it suffers small Bayes-regret relative to the Bayes-optimal policy for finite, discrete MDPs.

The second problem in Bayesian reinforcement learning is the choice of the prior distribution. This can be of critical importance for large or complex problems, for two reasons. Firstly, a well-chosen prior can lead to more efficient learning, especially in the finite-sample regime. Secondly, as reinforcement learning involves potentially unbounded interactions, the computational and space complexity of calculating posterior distributions, estimating marginals and performing sampling become extremely important. The choice of priors is the main focus of this paper. In particular, we introduce a prior  over piecewise-linear multivariate Gaussian models. This is based on the construction of a context tree model, using a cover tree structure, which defines a conditional distribution on local linear Bayesian multivariate models. Since  inference for the model can be done in closed form, the resulting algorithm is very efficient, in comparison with other non-parametric models such as Gaussian processes. The following section discusses how previous work is related to our model.

\subsection{Related work}
\label{sec:related-work}
One component in our model is the \emph{context tree}. Context trees were introduced by~\citet{willems:context} for sequential prediction~\citep[see][for an overview]{Begleiter:VOMM:JAIR:2004}. In this model, a distribution of variable order Markov models for binary sequences is constructed, where the tree distribution is defined through context-dependent weights (for probability of a node being part of the tree) and Beta distributions (for predicting the next observation). A recent extension to switching time priors~\citep{ErvenGrunwald:CatchingUp} has been proposed by \citet{veness2012context}. More related to this paper is an algorithm proposed by \citet{Kozat07} for prediction. This asymptotically converges to the best univariate piecewise linear model in a class of trees with fixed structure.

Many reinforcement learning approaches based on such trees have been proposed, but have mainly focused on the discrete partially observable case~\citep{Daswani2012,DBLP:journals/jair/VenessNHUS11,bellemare:bayesian-recursive:icml:2013,farias:universal}.\footnote{We note that another important work in tree-based reinforcement learning, though not directly related to ours, is that of~\citet{Ernst:TreeRL}, which uses trees for expected utility rather than model estimation.}
However, tree structures can generally be used to perform Bayesian inference in a number of other domains~\citep{paddock2003randomized,Meila+Jordan:mixtures-of-trees:jmlr:2001,wong2010optional}. 

The core of our model is a generalised context tree structure that defines a distribution on multivariate piecewise-linear-Gaussian models. Consequently, a necessary component in our model is a multivariate linear model at each node of the tree. Such models were previously used for Bayesian reinforcement learning in~\citep{ijcai:lbrl} and were shown to perform well relatively to least-square policy iteration (LSPI)~\citep{lagoudakis2003least}.  Other approaches using linear models include~\citep{strehl:linear:nips}, which proves mistake bounds on reinforcement learning algorithms using online linear regression, and \citep{abbeel2005exploration} who use separate linear models for each dimension. Another related approach in terms of structure is~\citep{brunskill:typed-models:jmlr}, which partitions the space into \emph{types} and estimates a simple additive model for each type. 

Linear-Gaussian models are naturally generalised by Gaussian processes (GP). Some examples of GP in reinforcement learning include~\citep{Rasmussen04, Deisenroth2009,Deisenroth2011:pilco}, which focused on a model-predictive approach, while the work of~\citet{Engel05} employed GPs for expected utility estimation. GPs are computationally demanding, in contrast to our tree-structured prior. Another problem with the cited GP-RL approaches is that they employ the marginal distribution in the dynamic programming step. This heuristic ignores the uncertainty about the model (which is implicitly taken into account in equations \ref{eq:bayes-utility}, \ref{eq:monte-carlo}). A notable exception to this is the policy gradient approach employed by \citet{ghavamzadeh:bpga} which uses full Bayesian quadrature. Finally, output dimensions are treated independently, which may not make good use of the data. Methods for efficient dependent GPs such as the one introduced by~\citet{alvarez2011efficient} have not yet been applied to reinforcement learning.

For decision making, this paper uses the simple idea of Thompson sampling~\citep{thompson1933lou,wyatt1998exploration,dearden98bayesian,strens2000bayesian}, which has been shown to be near-optimal in certain settings~\citep{Kaufmann:Thompson,agrawal:thompson,osband:thompson:nips:2013}. 
This avoids the computational complexity of building augmented MDP models~\citep{DBLP:conf/nips/AuerJO08,Asmuth:BOSS,castro2010smarter,araya2012near}, Monte-Carlo tree search~\citep{DBLP:journals/jair/VenessNHUS11}, sparse sampling~\citep{wang:bayesian-sparse-sampling:icml:2005}, stochastic branch and bound~\citep{dimitrakakis:icaart2010} or creating lower bounds on the Bayes-optimal value function~\citep{poupart2006asd,dimitrakakis:mmbi:ewrl:2011}. Thus the approach is reasonable as long as sampling from the model is efficient.

\subsection{Our contribution}
\label{sec:our-contribution}
Our approach is based upon three ideas. The first idea is to employ a cover tree~\citep{cover-tree:icml2006} to create a set of partitions of the state space. This avoids having to prespecify a structure for the tree. The second technical novelty is the introduction of an efficient non-parametric Bayesian conditional density estimator on the cover tree structure. This is a generalised context tree, endowed with a multivariate linear Bayesian model at each node. We use this to estimate the dynamics of the underlying environment. The multivariate models allow for a sample-efficient estimation by capturing dependencies. Finally, we take a sample from the posterior to obtain a piecewise linear Gaussian model of the dynamics. This can be used to generate policies. In particular, from this, we obtain trajectories of simulated experience, to perform approximate dynamic programming (ADP) in order to select a policy. Although other methods could be used to calculate optimal actions, we leave them for future work.

The main advantage of our approach is its generality and efficiency. The posterior calculation and prediction is fully conjugate and can be performed online. At the $t$-th time step, inference takes $O(\ln t)$ time. Sampling from the tree, which need only be done infrequently, is $O(t)$. These properties are in contrast to other non-parametric approaches for reinforcement learning such as GPs.  The most computationally heavy step of our algorithm is ADP. However, once a policy is calculated, the actions to be taken can be calculated in logarithmic time at each step. The specific ADP algorithm used is not integral to our approach and for some problems it might be more efficient to use an online algorithm.

\section{Cover Tree Bayesian RL}
\label{sec:cover-tree-bayesian}
The main idea of cover tree Bayesian reinforcement learning (CTBRL) is to construct a cover tree from the observations, simultaneously inferring a conditional probability density on the same structure, and to then use sampling to estimate a policy. We use a cover tree due to its efficiency compared with e.g. a fixed sequence of partitions or other dynamic partitioning methods such as KD-trees. The probabilistic model we use can be seen as a distribution over piecewise linear-Gaussian densities, with one local linear model for each set in each partition. Due to the tree structure, the posterior can be computed efficiently online.  By taking a sample from the posterior, we acquire a specific piecewise linear Gaussian model. This is then used to find an approximately optimal policy using approximate dynamic programming.

An overview of CTBRL is given in pseudocode in Alg.~\ref{alg:overview}. As presented, the algorithm works in an episodic manner.\footnote{An online version of the same algorithm (still employing Thompson sampling) would move line \ref{alg:ctbrl:adp} to just before line \ref{alg:ctbrl:act}. A fully Bayes online version would ``simply'' take an approximation of the Bayes-optimal action at every step.} When a new episode $k$ starts at time $t_k$, we calculate a new stationary policy by sampling a tree $\mdp_k$ from the current posterior $p_{t_k}(\mdp)$. This tree corresponds to a piecewise-linear model. We draw a large number of rollout trajectories from $\mdp_k$ using an arbitrary exploration policy. Since we have the model, we can use an initial state distribution that covers the space well. These trajectories are used to estimate a near-optimal policy $\pol_k$ using approximate dynamic programming. During the episode, we take new observations using $\pol_k$, while growing the cover tree as necessary and updating the posterior parameters of the tree and the local model in each relevant tree node. 

\begin{algorithm}
  \begin{algorithmic}[1]
    \STATE $k = 0$, $\pol_0 = \Uniform(\CA)$, prior $p_0$ on $\CM$.
    \FOR {$t=1, \ldots, T$}
    \IF {\texttt{episode-end}}
    \STATE $k := k + 1$.
    \STATE Sample model $\mdp_k \sim p_{t}(\mdp)$.
    \STATE Calculate policy $\pol_{k} \approx \argmax_\pol \E_{\mdp_k}^{\pol} U$. \label{alg:ctbrl:adp}
    \ENDIF
    \STATE Observe state $\vs_{t}$.
    \STATE Take action $a_t \sim \pol_k(\cdot \mid \vs_t).$ \label{alg:ctbrl:act}
    \STATE Observe next state $\vs_{t+1}$, reward $r_{t+1}$.
    \STATE Add a leaf node to the tree $\CT_{a_t}$, containing $\vs_{t}$.
    \STATE Update posterior: $p_{t+1}(\mdp) = p_t(\mdp \mid \vs_{t+1}, \vs_t, a_t)$ by updating the parameters of all nodes containing $\vs_{t}$.
    \ENDFOR
  \end{algorithmic}
  \caption{CTBRL (Episodic, using Thompson sampling)}
  \label{alg:overview}
\end{algorithm}
We now explain the algorithm in detail. First, we give an overview of the cover tree structure on which the context tree model is built. Then we show  how to perform inference on the context tree, while Section~\ref{sec:lbrl-model} describes the multivariate model used in each node of the context tree. The sampling approach and the approximate dynamic method are described in Sec.~\ref{sec:thompson}, while the overall complexity of the algorithm is discussed in Sec.~\ref{sec:complexity}.

\subsection{The cover tree structure}
\label{sec:cover-tree}
Cover trees are a data structure that can be applied to any metric space and are, among other things, an efficient method to perform nearest-neighbour search in high-dimensional spaces~\citep{cover-tree:icml2006}. In this paper, we use cover trees to automatically construct a sequence of partitions of the state space. Section~\ref{sec:cover-tree-prop} explains the properties of the constructed cover tree. As the formal construction duplicates nodes, in practice we use a reduced tree where every observed point corresponds to one node in the tree. This is explained in Section~\ref{sec:reduced-tree}. An explanation of how nodes are added to the structure is given in Section~\ref{sec:cover-tree-constr}.
 
\subsubsection{Cover tree properties.}
\label{sec:cover-tree-prop}
To construct a cover tree $\CT$ on a metric space $(\CZ, \psi)$ we require a set of points $D_t = \{\vz_1, \ldots, \vz_t\}$, with $\vz_i \in \CZ$, a metric $\psi$, and a constant $\zooming > 1$. We introduce a mapping function $[\cdot]$ so that the $i$-th tree node corresponds to one point $\pn{i}$ in this set.  The nodes are arranged in \emph{levels}, with each point being replicated at nodes in multiple levels, i.e. we may have $[i]=[j]$ for some $i \neq j$. Thus, a point corresponds to multiple nodes in the tree, but to \emph{at most one node} at any one level. Let $\CSet_n$ denote the set of points corresponding to the nodes at level $n$ of the tree and $\Children{i} \subset \CSet_{n-1}$ the corresponding set of children. 
If $i \in \CSet_n$ then the level of $i$ is $\level{i} = n$. The tree has the
following properties:
\begin{enumerate}
\item Refinement: $\CSet_n \subset \CSet_{n-1}$.
\item Siblings separation: $i, j \in \CSet_n$, $\metric{\pn{i}}{\pn{j}} > \zooming^n$.
\item Parent proximity: If $i \in \CSet_{n-1}$ then $\exists$ a unique $j \in \CSet_{n}$ such that $\metric{\pn{i}}{\pn{j}} \leq \zooming^n$ and $i \in \Children{j}$.
\end{enumerate}
These properties can be interpreted as follows. Firstly lower levels always contain more points. Secondly, siblings at a particular level are always well-separated. Finally, a child must be close to its parent. These properties directly give rise to the theoretical guarantees given by the cover tree structure, as well as methods for searching and adding points to the tree, as explained below.

\subsubsection{The reduced tree} 
\label{sec:reduced-tree}
As formally the cover tree duplicates nodes, in practice we use the \emph{explicit representation}~\citep[described in more detail in Sec.~2 of][]{cover-tree:icml2006}. This only stores the top-most tree node $i$ corresponding to a point $\pn{i}$. We denote this \emph{reduced tree} by $\hat{\CT}$. The \emph{depth} $\depth{i}$ of node $i \in \hat{\CT}$ is equal to its number of ancestors, with the root node having a depth of $0$. After $t$ observations, the set of nodes containing a point $\vz$, is: 
\begin{align}
  \RSet_t(\vz) \defn \cset{i \in \hat{\CT}}{\vz \in B_i},
\end{align}
where $B_i = \cset{\vz \in \CZ}{\metric{\pn{i}}{\vz} \leq \zooming^{\depth{i}}}$ is the neighbourhood of $i$. Then $\RSet_t(\vz)$ forms a path in the tree, as each node only has one parent, and can be discovered in logarithmic time through the \texttt{Find-Nearest} function~\citep[][The. 5]{cover-tree:icml2006}. This fact allows us to efficiently search the tree, insert new nodes, and perform inference.

\subsubsection{Inserting nodes in the cover tree}
\label{sec:cover-tree-constr}
The cover tree insertion we use is only a minor adaptation of the \texttt{Insert} algorithm by \citet{cover-tree:icml2006}. For each action $a \in \CA$, we create a different reduced tree $\hat{\CT}_a$, over the state space, i.e. $\CZ = \CS$, and build the tree using the metric $\metric{\vs}{\vs'} = \|\vs - \vs'\|_1$.  

At each point in time $t$, we obtain a new observation tuple $\vs_t, a_t, \vs_{t+1}$. We select the tree $\hat{\CT}_{a_t}$ corresponding to the action. Then, we traverse the tree, decreasing  $d$ and keeping a set of nodes $Q_d \subset G_d$ that are $\zooming^{d}$-close to $\vs_t$. We stop whenever $Q_d$ contains a node that would satisfy the \emph{parent proximity} property if we insert the new point at $d-1$, while the children of all other nodes in $Q_d$ would satisfy the \emph{sibling separation} property. This means that we can now insert the new datum as a child of that node.\footnote{The exact implementation is available in the CoverTree class in \cite{beliefbox}.} Finally, the next state $\vs_{t+1}$ is only used during the inference process, explained below.

\subsection{Generalised context tree inference}
\label{sec:ct-model}
In our model, each node $i \in \hat{\CT}$ is associated with a particular Bayesian model. The main problem is how to update the individual models and how to combine them. Fortunately, a closed form solution exists due to the tree structure. We use this to define a \emph{generalised context tree}, which can be used for inference. 

As with other tree models~\citep{willems:context,Ferguson:Prior-Polya:1974}, our model makes predictions by marginalising over a set of simpler models. Each node in the context tree is called a \emph{context}, and each context is associated with a specific local model. At time $t$,  given an observation $\vs_t = \vs$ and an action $a_t = a$, we calculate the marginal (predictive) density $\pt$ of the next observation:
\begin{equation}
  \pt(\vs_{t+1} \mid \vs_t, a_t)
  =
  \sum_{c_t} 
  \pt(\vs_{t+1} \mid \vs_t, c_t)
  \pt(c_t \mid \vs_t, a_t),
\end{equation}
where we use the symbol $\pt$ throughout for notational simplicity to denote marginal distributions from our posterior at time $t$.
Here, $c_t$ is such that if $\pt(c_t = i \mid \vs_t, a_t) > 0$, then the current state is within the neighbourhood of $i$-th node of the reduced cover tree $\hat{\CT}_{a_t}$, i.e. $\vs_t \in B_i$.

For Euclidean state spaces, the $i$-th component density $\pt(\vs_{t+1} \mid \vs_t, c_t = i)$ employs a linear Bayesian model, which we describe in the next section.
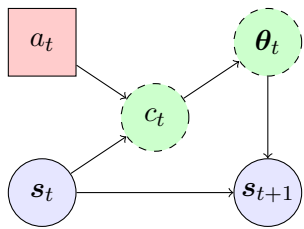
\begin{figure}[t!]
  \centering
  \begin{tikzpicture}
    \node[observed] at (0,0) (s1) {$\vs_{t}$};
    \node[observed] at (3,0) (s2) {$\vs_{t+1}$};
    \node[select] at (0,2) (a1) {$a_{t}$};
    \node[hidden] at (1.5,1) (c1) {$c_{t}$};
    \node[hidden] at (3,2) (t1) {$\vt_{t}$};
    \draw[->] (s1) edge (s2);
    \draw[->] (s1) edge (c1);
    \draw[->] (a1) edge (c1);
    \draw[->] (c1) edge (t1);
    \draw[->] (t1) edge (s2);
  \end{tikzpicture}
  \caption{The generalised context tree graphical model. Blue circles indicate observed variables. Green dashed circles indicate latent variables. Red rectangles indicate choice variables. Arrows indicate dependencies. Thus, the context distribution at time $t$ depends on both the state and action, while the parameters depend on the context. The next state depends on the action only indirectly.}
  \label{fig:context-tree-graph}
\end{figure}
The graphical structure of the model is shown in simplified form in Fig.~\ref{fig:context-tree-graph}. The context at time $t$ depends only on the current state $\vs_t$ and action $a_t$. The context corresponds to a particular local model with parameter $\vt_t$, which defines the conditional distribution.

The probability distribution $\pt(c_t \mid \vs_t, a_t)$ is determined through \emph{stopping probabilities}. More precisely, we set it be equal to the probability of stopping at the $i$-th context, when performing a walk from the leaf node containing the current observation towards the root, stopping at the $j$-th node with probability $w_{j,t}$ along the way:
\begin{equation}
  \label{eq:stopping}
  \pt(c_t = i \mid \vs_t, a_t)
  =
  w_{i,t} \prod_{j \in \Descendants{i}} (1 - w_{j,t}),
\end{equation}
where $\Descendants{i}$ are the descendants of $i$ that contain the observation $\vs_t$. This forms a path from $i$ to the leaf node containing $\vs_t$. Note that $w_{0,t}=1$, so we always stop whenever we reach the root. Due to the effectively linear structure of the relevant tree nodes, the stopping probability parameters $w$ can be updated in closed form, as shown in \citep[Theorem 1]{dimitrakakis:aistats:2010} via Bayes' theorem as follows:
\begin{equation}
  w_{i,t+1}
  = 
  \frac{\pt(\vs_{t+1} \mid \vs_t, c_t = i) w_{i,t}}
  {\pt(\vs_{t+1} \mid \vs_t, c_t \in  \{i\} \cup \Descendants{i})}.
  \label{eq:weight-update}
\end{equation}
Since there is a different tree for each action,  $c_ti$ uniquely identifies a tree, the action does not need to enter in the conditional expressions above. Finally, it is easy to see, by marginalisation and the definition of the stopping probabilities, that the denominator in the above equation can be calculated recursively:
\begin{multline}
  \pt(\vs_{t+1} \mid \vs_t, c_t \in \{i\} \cup \Descendants{i})
  =
  w_{i,t}
  \pt(\vs_{t+1} \mid \vs_t, c_t = i)
  +
  (1 - w_{i,t})
  \pt(\vs_{t+1} \mid \vs_t, c_t \in \Descendants{i}).
  \label{eq:denominator}
\end{multline}
Consequently, inference can be performed with a simple forward-backward sweep through a single tree path. In the forward stage, we compute the probabilities of the denominator, until we reach the point where we have to insert a new node. Whenever a new node is inserted in the tree, its weight parameter is initialised to $2^{-\depth{i}}$. We then go backwards to the root node, updating the weight parameters and the posterior of each model. The only remaining question is how to calculate the individual predictive marginal distributions for each context $i$ in the forward sweep and how to calculate their posterior in the backward sweep. In this paper, we associate a linear Bayesian model with each context, which provides this distribution.

\subsection{The linear Bayesian model}
\label{sec:lbrl-model}
In our model we assume that, given $c_t  = i$, the next state $\vs_{t+1}$ is given by a linear transformation of the current state and additive noise $\noise_{i,t}$:
\begin{align}
  \vs_{t+1} &= \MA_i \vx_t + \noise_{i,t},
  &
  \vx_t
  &\defn
  \begin{pmatrix}
    \vs_t\\
    1
  \end{pmatrix},
  \label{eq:linear-model}
\end{align}
where  $\vx_t$ is the current state vector augmented by a unit basis.\footnote{While other transformations of $\vs_t$ are possible, we do not consider them in this paper.} In particular, each context models the dynamics via a Bayesian multivariate linear-Gaussian model. For the $i$-th context, there is a different (unknown) parameter pair $(\MA_i, \MV_i)$ where $\MA_i$ is the \emph{design} matrix and $\MV_i$ is the \emph{covariance} matrix. Then the next state distribution is:
\begin{align}
  \vs_{t+1} &\mid \vx_t=\vx, c_t = i \sim \Normal(\MA_i \vx, \MV_i).
\end{align}
Thus, the parameters $\vt_t$ which are abstractly shown in Fig.~\ref{fig:context-tree-graph} correspond to the two matrices $\MA, \MV$.
We now define the conditional distribution of these matrices given $c_t = i$.

We can model our uncertainty about these parameters with an appropriate prior distribution $\bel_0$. In fact, a conjugate prior exists in the form of the \emph{matrix inverse-Wishart normal} distribution.
In particular, given $\MV_i = \MV$, the distribution for $\MA_i$ is matrix-normal, while the marginal distribution of $\MV_i$ is inverse-Wishart:
\begin{align}
  \MA_i \mid \MV_i = \MV \sim \mbel(\MA_i &\mid \underbrace{\pmean, \pcov}_{\mathclap{\text{prior parameters}}},
  \MV)
  \label{eq:sample-mean}
  \\
  \MV_i \sim \pbel(\MV_i &\mid \overbrace{\MW, n}).
  \label{eq:sample-covariance}
\end{align}
Here $\mbel$ is the prior on design matrices, which has a matrix-normal distribution, conditional on the covariance and two prior parameters: $\pmean$, which is the prior mean and $\pcov$ which is the prior covariance of the dependent variable (i.e. the output). Finally, $\pbel$ is the marginal prior on covariance matrices, which has an inverse-Wishart distribution with $\pwish$ and $n$. More precisely, the distributions have the following forms:
\begin{align*}
  \mbel(\MA_i \mid \pmean, \pcov, \MV)
  &\propto
  e^{-\frac{1}{2} \trace\left[\trans{(\MA_i - \pmean)}\MV^{-1}(\MA_i - \pmean) \pcov \right]}
  \\
  \pbel(\MV \mid \pwish, \porder)
  &\propto
  |\MV^{-1} \pwish/2|^{n/2}  e^{-\frac{1}{2} \trace(\MV^{-1} \pwish)}.
\end{align*}
Essentially, the model extends the classic Bayesian linear
regression model~\citep[e.g][]{Degroot:OptimalStatisticalDecisions} to the multivariate case via vectorisation of the mean matrix. Since the prior
  is conjugate, it is relatively simple to calculate the posterior after each observation. For simplicity, and to limit the total number of prior parameters we have to select, we use the same prior parameters $(\pmean_i, \pcov_i, \pwish_i, \porder_i)$ for all contexts in the tree.


To integrate this with inference in the tree, we must define the marginal distribution used in the nominator of \eqref{eq:weight-update}. This is a multivariate Student-$t$ distribution, so if the posterior parameters for context $i$ at time $t$ are $(\pmean_i^t, \pcov_i^t, \pwish_i^t, \porder_i^t)$, then this is:
\begin{equation}
  \label{eq:student}
  p_t(\vs_{t+1} \mid \vx_t = \vx, c_t = i)
  =
  \Student(\pmean_i^t, \pwish_i^t/z_i^t, 1 + \porder_i^t),
\end{equation}
where $z_i^t = 1 - \trans{\vx}(\pcov_i^t + \vx \trans{\vx})^{-1} \vx$.


\subsubsection{Regression illustration}
\definecolor{mycolor1}{rgb}{0.9,0.9,0.9}
\definecolor{mycolor2}{rgb}{0.3,0.3,1.0}
\definecolor{mycolor3}{rgb}{0.3,1.0,0.3}
\definecolor{mycolor4}{rgb}{1.0,0.3,0.3}
\begin{figure}[ht]
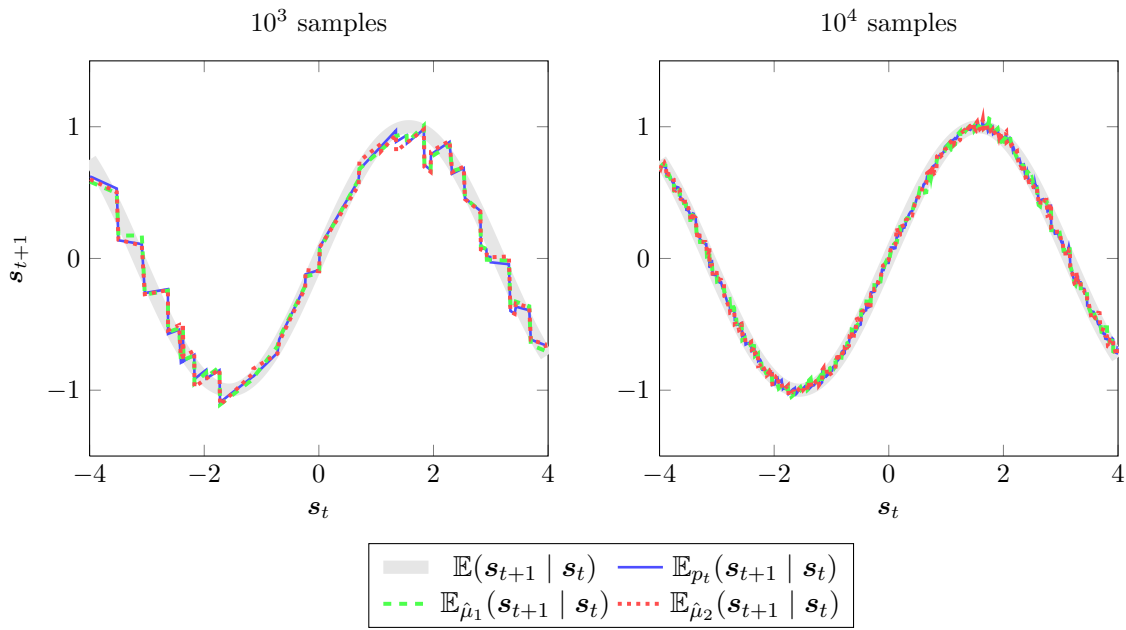

  \centering{
    \subfigure{\small
      \hspace{-5em}
      \input{figures/regression_1000.tikz}
      \hfill
      \input{figures/regression_10000.tikz}
    }
    \ref{leg:regression}
  }
  \caption{Regression illustration. We plot the expected value for the real distribution, the marginal, as well as two sampled models $\hat{\mdp}_1, \hat{\mdp}_2 \sim p_t(\mdp)$.}
  \label{fig:regression}
\end{figure}
 An illustration of inference using the generalised context tree is given in Fig.~\ref{fig:regression}, where the piecewise-linear structure is evident. The $\vs_t$ variates are drawn uniformly in the displayed interval, while $\vs_{t+1} \mid \vs_t = \vs \sim \Normal(\sin(\vs), 0.1)$, i.e. drawn a normal distribution with mean $\sin(\vs_t)$ and variance $0.1$. The plot shows the marginal expectation $\E_{\pt}$, as well as the expectation from two different models sampled from the posterior $\pt(\mu)$.

\subsection{Approximating the optimal policy with Thompson sampling}
\label{sec:thompson}
Many algorithms exist for finding the optimal policy for a specific MDP $\mdp$, or for calculating the expected utility of a given policy for that MDP. Consequently, a simple idea is to draw MDP samples $\mdp_i$ from the current posterior distribution and then calculate the expected utility of each. This can be used to obtain approximate lower and upper bounds on the Bayes-optimal expected utility by maximising over the set of memoryless policies $\Pols_1$. Taking $K$ samples, allows us to calculate the upper and lower bounds with accuracy $O(1/\sqrt{K})$.
\begin{align}
  \max_{\pol \in \Pols_1} \E_{\bel}^\pol U 
  &\approx \max_{\pol \in \Pols_1} \frac{1}{K} \sum_{i=1}^K \E_{\mdp_i}^{\pol} U
  \leq \frac{1}{K} \sum_{i=1}^K \max_{\pol \in \Pols_1} \E_{\mdp_i}^{\pol} U,
  &
  \mdp_i \sim \bel_t(\mdp).
  \label{eq:monte-carlo}
\end{align}
We consider only the special case $K = 1$, i.e. when we only sample a single MDP. Then the two values are identical and we recover Thompson sampling. The main problems we have to solve now is how to sample a model and how to calculate a policy for the sampled model.

\subsubsection{Sampling a model from the posterior. }
Each model $\mdp$ sampled from the posterior corresponds to a particular choice of tree parameters. Sampling is done in two steps. The first generates a partition from the tree distribution and the second step generates a linear model for each context in the partition.

The first step is straightforward. We only need to sample a set of weights $\hat{w}_{i} \in \{0,1\}$ such that $\Pr(\hat{w}_{i} = 1) = w_{i,t}$, as shown in \citep[][Rem.~2]{dimitrakakis:aistats:2010}. This creates a \emph{partition}, with one Bayesian multivariate linear model responsible for each context in the partition.

The second step is to sample a design and covariance matrix pair $(\hat{\MA}_i, \hat{\MV}_i)$ for each context $i$ in the partition. This avoids sampling matrices for contexts not part of the sampled tree. As the model suggests, we can first sample the noise covariance by plugging the posterior parameters in \eqref{eq:sample-covariance} to obtain $\hat{\MV}_i$. Sampling from this distribution can be done efficiently using the algorithm suggested by~\cite{smith1972wishart}. We then plug in $\hat{\MV}_i$ into the conditional design matrix posterior \eqref{eq:sample-mean} to obtain a design matrix $\hat{\MA}_i$ by sampling from the resulting matrix-normal distribution.

The final MDP sample $\mdp$ from the posterior has two elements. Firstly, a set of contexts $\hat{C}^\mdp \subset \bigcup_{a \in \CA} \hat{\CT_a}$, from all action trees. This set is a partition with associated mapping $\ctx^\mdp : \CS \times \CA \to \hat{C}^\mdp$. Secondly, a set of associated design and covariance matrices $\cset{(A^\mdp_i, V^\mdp_i)}{i \in \hat{C}^\mdp}$ for each context. Then the prediction of the sampled MDP is:
\begin{equation}
  \Pr_\mdp(\vs_{t+1} \mid \vs_t, a_t) = \Normal(A^\mdp_{f(\vs_t, a_t)} \vx_t, V^\mdp_{f(\vs_t, a_t)}),
  \label{eq:prediction}
\end{equation}
where $\vx_t$ is given in \eqref{eq:linear-model}.
\subsubsection{Finding a policy for a sample via ADP}
\label{sec:api}
In order to calculate an optimal policy $\pol^*(\mdp)$ for $\mdp$, we generate a large number of trajectories from $\mdp$ using a uniform policy. After selecting an appropriate set of basis functions, we then employ a variant of the least-squares policy iteration (LSPI~\citep{lagoudakis2003least}) algorithm, using least-squares temporal differences (LSTD~\cite{bradtke1996linear}) rather than LSTDQ. This is possible because since we have $\mdp$ available, we have access to (\ref{eq:prediction}) and it makes LSPI slightly more efficient. 

More precisely, consider the \emph{value function} $V_\mdp^\pol : \CS \to \Reals$, defined as:
\begin{align}
  \label{eq:value-function}
  V_\mdp^\pol(\vs) \defn \E_{\mdp}^{\pol} \left(U \mid \vs_t = \vs\right).
\end{align}
Unfortunately, for continuous $\CS$ finding an optimal policy requires approximations. A common approach is to make use of the fact that:
\begin{equation}
  \label{eq:bellman}
  V_\mdp^\pol(\vs) = \rho(\vs)
  + \disc \int_{\CS} V_\mdp^\pol(\vs') \dd{P}_{\mdp}^{\pol}(\vs'\mid\vs),
\end{equation}
where we assume for simplicity that $\rho(\vs)$ is the reward obtained at state $\vs$. The conditional measure
$P^\pol_\mdp$ is the transition kernel on $\CS$ induced by $\mdp, \pol$, introduced in Section~\ref{sec:setting}.  We then select a parametric family $v_\omega : \CS \to \Reals$ with parameter $\omega \in \Omega$ and minimise:
\begin{equation}
  \label{eq:api}
  h(\omega) + \int_\CS \norm{v_\omega(\vs) - \rho(\vs) - \disc \int_\CS v_\omega(\vs') \dd{\hat{P}}_{\mdp}^{\pol}(\vs'|\vs)} \dd{\chi}(\vs),
\end{equation}
where $h$ is a regularisation term, $\chi$ is an appropriate measure on $\CS$ and $\hat{P}_\mdp^\pol$ is an empirical estimate of the transition kernel, used to approximate the respective integral that uses $P_\mdp^\pol$. As we can take an arbitrary number of trajectories from $\mdp, \pol$, this can be as accurate as our computational capacity allows. 

In practice, we minimise \eqref{eq:api} with a generalised linear model (defined on an appropriate basis) for $v_\omega$ while $\chi$ need only be positive on a set of representative states. Specifically, we employ a variant of the least-squares policy iteration (LSPI~\citep{lagoudakis2003least}) algorithm, using the least-squares temporal differences (LSTD~\cite{bradtke1996linear}) for the minimisation of \eqref{eq:api}. Then the norm is the euclidean norm and the regularisation term is $h(\omega) = \lambda \|\omega\|$. In order to estimate the inner integral, we take $\kLSTD \geq 1$ samples from the model so that
\begin{align}
\hat{P}_{\mdp}^{\pol}(\vs' \mid \vs)
&\defn
\frac{1}{\kLSTD}
\sum_{i=1}^{\kLSTD} \ind{\vs^i_{t+1} = \vs' \mid \vs^i_t = \vs},
\label{eq:lstd-approximation}
\\
\vs^{i}_{t+1} & \mid \vs^{i}_t = \vs \sim P^\pol_\mdp(\cdot \mid \vs), \nonumber
\end{align}
where $\ind{\cdot}$ is an indicator function and  
$P_{\mdp}^{\pol}$ is decomposable in known terms. Equation \eqref{eq:lstd-approximation} is also used for action selection in order to calculate an approximate expected utility $q_\omega(\vs, a)$ for each state-action pair $(\vs, a)$:
\begin{align}
  q_\omega(\vs, a) \defn \rho(s) + \disc \int_{\mathrlap{\CS}} v_\omega(\vs') \dd{\hat{P}}_{\mdp}^{\pol}(\vs'|\vs)
  \label{eq:approximate-q}
\end{align}
Effectively, this approximates the integral via sampling. This may add a small amount\footnote{Generally, this error is bounded by $O(\kLSTD^{-1/2})$.} of additional stochasticity to action selection, which can be reduced\footnote{We remind the reader that Thompson sampling itself results in considerable exploration by sampling an MDP from the posterior. Thus, additional randomness may be detrimental.} by increasing $\kLSTD$.

Finally, we optimise the policy by approximate policy iteration. At the $j$-th iteration we obtain an improved policy $\hat{\pol}_j(a \mid \vs) \propto \Pr[a \in \argmax_{a' \in \CA} q_{\omega_{j-1}}(\vs, a')]$ from $\omega_{j-1}$ and then estimate $\omega_j$ for the new policy. 

\subsection{Complexity}
\label{sec:complexity}
We now analyse the computational complexity of our approach, including the online complexity of inference and decision making, and of the sampling and ADP taking place every episode. It is worthwhile to note two facts. Firstly, that the complexity bounds related to the cover tree depend on a constant $c$, which however depends on the distribution of samples in the state space. In the worst case (i.e. a uniform distribution), this is bounded exponentially in the dimensionality of the actual state space. While we do not expect this to be the case in practice, it is easy to construct a counterexample where this is the case. Secondly, that the complexity of the ADP step  is largely independent of the model used, and mostly depends on the number of trajectories we take in the sampled model and the dimensionality of the feature space.

First, we examine the total computation time that is required to construct the tree.
\begin{corollary}
  Cover tree construction from $t$ observations takes $O(t \ln t)$ operations.
\end{corollary}
\begin{proof}
  In the cover tree, node insertion and query are $O(\ln t)$\citep[][Theorems 5, 6]{cover-tree:icml2006}.
  Then note that $\sum_{k=1}^t \ln k \leq \sum_{k=1}^t \ln t = t \ln t$.
\end{proof}
At every step of the process, we must update our posterior parameters. Fortunately, this also takes logarithmic time as we only need to perform calculations for a single path from the root to a leaf node.
\begin{lemma}
  If $\CS \subset \Reals^\dimS$, then inference at time step $t$ has complexity $O(\dimS^3 \ln t)$.
\end{lemma}
\begin{proof}
  At every step, we must perform inference on a number of nodes equal to the length of the path containing the current observation. This is bounded by the depth of the tree, which is in turn bounded by $O(\ln t)$ from \citep[][Lem. 4.3]{cover-tree:icml2006}. Calculating \eqref{eq:weight-update} is linear in the depth. For each node, however, we must update the linear-Bayesian model, and calculate the marginal distribution. Each requires inverting an $\dimS \times \dimS$ matrix, which has complexity $O(\dimS^3)$. 
\end{proof}
Finally, at every step we must choose an action through value function look-up. This again takes logarithmic time, but there is a scaling depending on the complexity of the value function representation.
\begin{lemma}
  If the LSTD basis has dimensionality $\dimLSTD$, then taking a decision at time $t$ has complexity $O(\kLSTD \dimLSTD \ln t)$.
  \label{lem:lstd-decision}
\end{lemma}
\begin{proof}
  To take a decision we merely need to search in each action tree to find a corresponding path. This takes $O(\ln t)$ time for each tree. After Thompson sampling, there will only be one linear model for each action tree. LSTD takes $\kLSTD$ operations, and requires the inner product of two $\dimLSTD$-dimensional vectors.
\end{proof}
The above lemmas give the following result:
\begin{theorem}
  At time $t$, the online complexity of CTBRL is $O((\dimS^3 + \kLSTD \dimLSTD) \ln t)$.
  \label{the:online-complexity}
\end{theorem}
We now examine the complexity of finding a policy. Although this is the most computationally demanding part, its complexity is not dependent on the cover tree structure or the probabilistic inference method used. However, we include it here for completeness.
\begin{lemma}
  Thompson sampling at time $t$ is $O(t \dimS^3)$.
\end{lemma}
\begin{proof}
  In the worst case, our sampled tree will contain all the leaf nodes of the reduced tree, which are $O(t)$. For each sampled node, the most complex operation is Wishart generation, which is $O(\dimS^3)$~\citep{smith1972wishart}.
\end{proof}
\begin{lemma}
  If we use $\nsam$ samples for LSTD estimation
  and the basis dimensionality is  $\dimLSTD$, this step has complexity $O(\dimLSTD^3 + \nsam (\dimLSTD^2 + \kLSTD \dimLSTD \ln t))$.  
  \label{lem:lstd-estimation}
\end{lemma}
\begin{proof}
  For each sample we must take a decision according to the last policy, which requires $O(\kLSTD \dimLSTD \ln t)$ as shown previously. We also need to update two matrices (see \cite{boyan2002technical}), which is $O(\dimLSTD^2)$. So, $O(\nsam (\dimLSTD^2 + \kLSTD \dimLSTD \ln t)  )$ computations  must be performed for the total number of the selected samples. Since LSTD requires an $\dimLSTD \times \dimLSTD$ matrix inversion, with complexity $O(\dimLSTD^3)$, we obtain the final result.
\end{proof}
From Lemmas~\ref{lem:lstd-decision} and \ref{lem:lstd-estimation} it follows that:
\begin{theorem}
  If we employ API with $\kAPI$ iterations, the total complexity of calculating a new policy is $O(t\dimS^3 + \kAPI (\dimLSTD^3 + \nsam (\dimLSTD^2 + \kLSTD \dimLSTD \ln t)))$.
  \label{the:api-complexity}
\end{theorem}

Thus, while the online complexity of CTBRL is only logarithmic in $t$, there is a substantial cost when calculating a new policy. This is only partially due to the complexity of sampling a model, which is manageable when the state space has small dimensionality. Most of the computational effort is taken by the API procedure, at least as long as $t < (\dimLSTD/\dimS)^3$. However, we think this is unavoidable no matter what the model used is.

The complexity of Gaussian process (GP) models is substantially higher. In the simplest model, where each output dimension is modelled independently, inference is $O(\dimS t^3)$, while the fully multivariate tree model has complexity $O(\dimS^3 t \ln t)$. Since there is no closed form method for sampling a function from the process, one must resort to iterative sampling of points. For $n$ points, the cost is approximately $O(n \dimS t^3)$, which makes sampling long trajectories prohibitive. For that reason, in our experiments we only use the mean of the GP.

\section{Experiments}
\label{sec:experiments}
We conducted two sets of experiments to analyse the offline and the online performance. We compared CTBRL with the well-known LSPI algorithm~\citep{lagoudakis2003least} for the offline case, as well as an online variant~\citep{Busoniu10OnlineLSPI} for the online case. We also compared CTBRL with linear Bayesian reinforcement learning \citep[LBRL,][]{ijcai:lbrl} and finally GP-RL, where we simply replaced the tree model with a Gaussian process. For CTBRL and LBRL we use Thompson sampling. However, since Thompson sampling cannot be performed on GP models, we use the mean GP instead. In order to compute policies given a model, all model-based methods use the variant of LSPI explained in Section~\ref{sec:api}.
Hence, the only significant difference between each approach is the model used, and whether or not they employ Thompson sampling.

A significant limitation of Gaussian processes is that their computational complexity becomes prohibitive as the number of samples becomes extremely large. In order to make the GP model computationally practical, the greedy approximation approach introduced by~\citet{Engel02} has been adopted. This is a kernel sparsification methodology which incrementally constructs a dictionary of the most representative states. More specifically, an \emph{approximate linear dependence} analysis is performed in order to examine whether a state can be approximated sufficiently as a linear combination of current dictionary members or not.

We used one preliminary run and guidance from the literature to make an initial selection of possible hyper-parameters, such as the number of samples and the features used for LSTD and LSTD-$Q$. We subsequently used $10$ runs to select a single hyper-parameter combination for each algorithm-domain pair. The final evaluation was done over an independent set of $100$ runs.

For CTBRL and the GP model, we had the liberty to draw an arbitrary number of trajectories for the value function estimation. We drew $1$-step transitions from a set of $3000$ uniformly drawn states from the \emph{sampled} model (the mean model in the GP case). We used 25 API iterations on this data.

For the offline performance evaluation, we first drew rollouts from $k = \{10,20, \ldots , 50,$ $100, \ldots ,1000\}$ states drawn from the \emph{true environment's} starting distribution, using a uniformly random policy. The maximum horizon of each rollout was set equal to $40$. The collected data was then fed to each algorithm in order to produce a policy. This policy was evaluated over $1000$ rollouts on the environment. 

In the online case, we simply use the last policy calculated by each algorithm at the end of the last episode, so there is no separate learning and evaluation phase. This means that efficient exploration must be performed. For CTBRL, this is done using Thompson sampling. For online-LSPI, we followed the approach of~\citep{Busoniu10OnlineLSPI}, who adopts an $\epsilon$-greedy exploration scheme with an exponentially decaying schedule $\epsilon_t = \epsilon_d^t$, with $\epsilon_0 = 1$. In preliminary experiments, we found $\epsilon_d = 0.997$ to be a reasonable compromise. We compared the algorithms online for $1000$ episodes. 

\begin{figure}[ht]
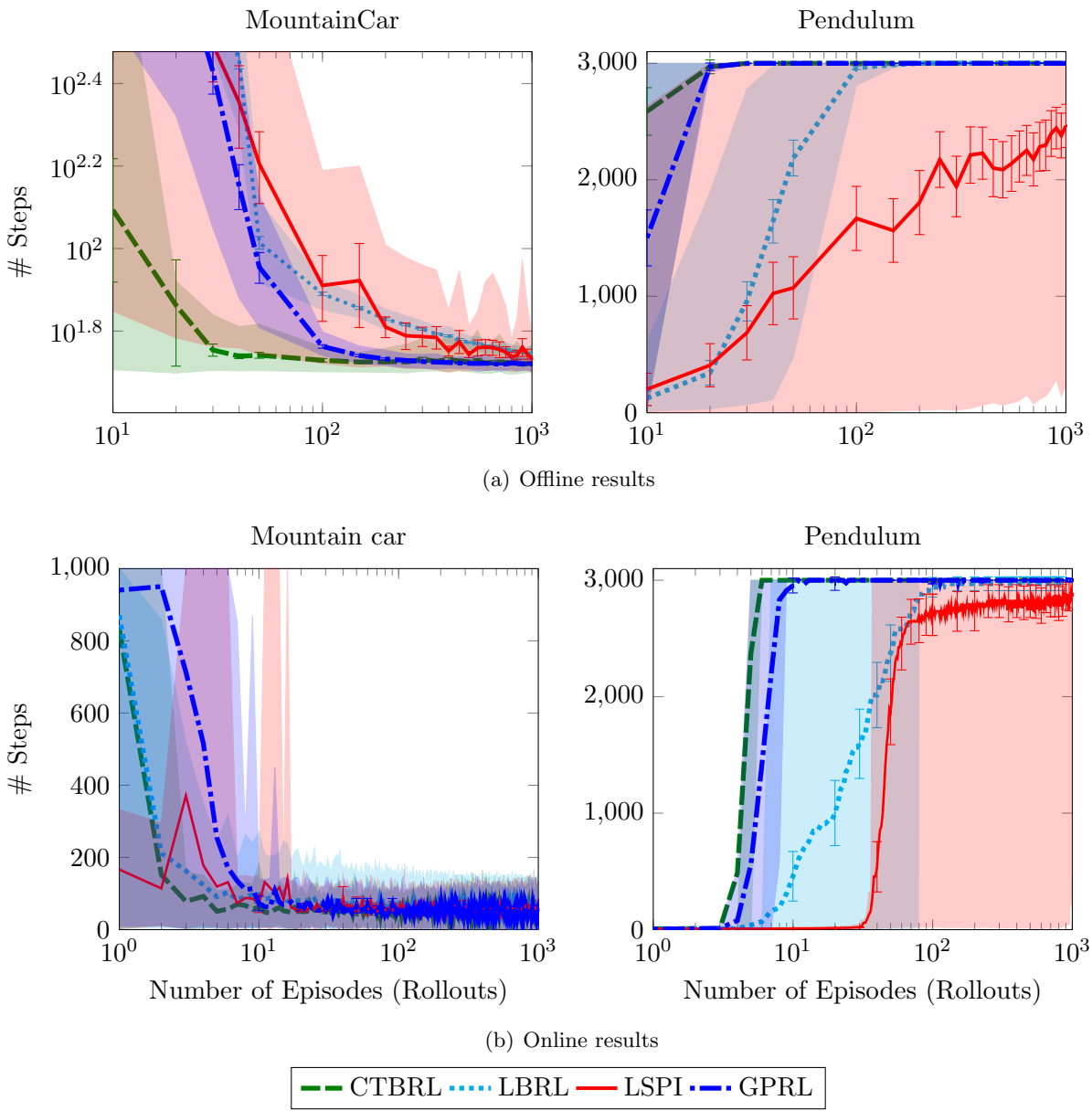

  \begin{center}  
    \subfigure[Offline results]{
    	  \hspace{-2em}
      \input{figures/Mountain_Offline_GP_LBRL.tikz}
      \hfill
%
%
\begin{tikzpicture}

\begin{semilogxaxis}[%
legend columns=-1,
legend entries={CTBRL, LBRL, LSPI, GPRL},
legend to name=leg:offline,
title={Pendulum},
width=0.4\columnwidth,
scale only axis,
xmin=10,
xmax=1000,
ymin=0,
ymax=3100
]
\addplot [
color=green!50!black,
dash pattern=on 8pt off 2pt on 8pt off 2pt,
line width=2.0pt
]
table[row sep=crcr]{
10 2586.83251\\
20 2970.14087\\
30 3000\\
40 3000\\
50 3000\\
100 3000\\
150 3000\\
200 3000\\
250 3000\\
300 3000\\
350 3000\\
400 3000\\
450 3000\\
500 3000\\
550 3000\\
600 3000\\
650 3000\\
700 3000\\
750 3000\\
800 3000\\
850 3000\\
900 3000\\
950 3000\\
1000 3000\\
};
\addplot [
color=green!50!black,
line width=1.0pt,
only marks,
mark=o,
mark options={scale=0.0},
forget plot
]
plot [error bars/.cd, y dir = both, y explicit]
coordinates{
(10,2586.83251) +- (0.0,204.401760563742)(20,2970.14087) +- (0.0,59.246991905623)(30,3000) +- (0.0,0)(40,3000) +- (0.0,0)(50,3000) +- (0.0,0)(100,3000) +- (0.0,0)(150,3000) +- (0.0,0)(200,3000) +- (0.0,0)(250,3000) +- (0.0,0)(300,3000) +- (0.0,0)(350,3000) +- (0.0,0)(400,3000) +- (0.0,0)(450,3000) +- (0.0,0)(500,3000) +- (0.0,0)(550,3000) +- (0.0,0)(600,3000) +- (0.0,0)(650,3000) +- (0.0,0)(700,3000) +- (0.0,0)(750,3000) +- (0.0,0)(800,3000) +- (0.0,0)(850,3000) +- (0.0,0)(900,3000) +- (0.0,0)(950,3000) +- (0.0,0)(1000,3000) +- (0.0,0)};

\addplot[solid,fill=green!50!black,opacity=2.000000e-01,draw=green!50!black,forget plot]
table[row sep=crcr]{
x y\\
10 3000 \\
20 3000 \\
30 3000 \\
40 3000 \\
50 3000 \\
100 3000 \\
150 3000 \\
200 3000 \\
250 3000 \\
300 3000 \\
350 3000 \\
400 3000 \\
450 3000 \\
500 3000 \\
550 3000 \\
600 3000 \\
650 3000 \\
700 3000 \\
750 3000 \\
800 3000 \\
850 3000 \\
900 3000 \\
950 3000 \\
1000 3000 \\
1000 3000 \\
950 3000 \\
900 3000 \\
850 3000 \\
800 3000 \\
750 3000 \\
700 3000 \\
650 3000 \\
600 3000 \\
550 3000 \\
500 3000 \\
450 3000 \\
400 3000 \\
350 3000 \\
300 3000 \\
250 3000 \\
200 3000 \\
150 3000 \\
100 3000 \\
50 3000 \\
40 3000 \\
30 3000 \\
20 3000 \\
10 11.298 \\
};

\addplot [
color=cyan,
dotted,
line width=2.0pt
]
table[row sep=crcr]{
10 123.88523\\
20 342.52202\\
30 956.25758\\
40 1643.15867\\
50 2187.41577\\
100 2960.22154\\
150 2987.77889\\
200 2999.213\\
250 2999.89718\\
300 2999.97106\\
350 3000\\
400 2999.98596\\
450 3000\\
500 3000\\
550 3000\\
600 3000\\
650 3000\\
700 3000\\
750 3000\\
800 3000\\
850 3000\\
900 3000\\
950 3000\\
1000 2999.97571\\
};
\addplot [
color=cyan,
line width=1.0pt,
only marks,
mark=o,
mark options={scale=0.0},
forget plot
]
plot [error bars/.cd, y dir = both, y explicit]
coordinates{
(10,123.88523) +- (0.0,61.1619462895461)(20,342.52202) +- (0.0,105.492360493034)(30,956.25758) +- (0.0,168.390647069662)(40,1643.15867) +- (0.0,186.486365867352)(50,2187.41577) +- (0.0,153.195679963919)(100,2960.22154) +- (0.0,13.0920403811924)(150,2987.77889) +- (0.0,9.13939630676259)(200,2999.213) +- (0.0,0.380578503243214)(250,2999.89718) +- (0.0,0.127225259274383)(300,2999.97106) +- (0.0,0.0574232385788491)(350,3000) +- (0.0,0)(400,2999.98596) +- (0.0,0.0278584060001776)(450,3000) +- (0.0,0)(500,3000) +- (0.0,0)(550,3000) +- (0.0,0)(600,3000) +- (0.0,0)(650,3000) +- (0.0,0)(700,3000) +- (0.0,0)(750,3000) +- (0.0,0)(800,3000) +- (0.0,0)(850,3000) +- (0.0,0)(900,3000) +- (0.0,0)(950,3000) +- (0.0,0)(1000,2999.97571) +- (0.0,0.0481966297538747)};

\addplot[solid,fill=cyan,opacity=2.000000e-01,draw=cyan,forget plot]
table[row sep=crcr]{
x y\\
10 604.047 \\
20 1892.441 \\
30 2779.38 \\
40 2979.974 \\
50 2986.343 \\
100 3000 \\
150 3000 \\
200 3000 \\
250 3000 \\
300 3000 \\
350 3000 \\
400 3000 \\
450 3000 \\
500 3000 \\
550 3000 \\
600 3000 \\
650 3000 \\
700 3000 \\
750 3000 \\
800 3000 \\
850 3000 \\
900 3000 \\
950 3000 \\
1000 3000 \\
1000 3000 \\
950 3000 \\
900 3000 \\
850 3000 \\
800 3000 \\
750 3000 \\
700 3000 \\
650 3000 \\
600 3000 \\
550 3000 \\
500 3000 \\
450 3000 \\
400 3000 \\
350 3000 \\
300 3000 \\
250 3000 \\
200 2994.702 \\
150 2960.292 \\
100 2812.21 \\
50 474.952 \\
40 111.24 \\
30 70.955 \\
20 33.348 \\
10 14.141 \\
};

\addplot [
color=red,
solid,
line width=1.5pt
]
table[row sep=crcr]{
10 201.20145\\
20 407.62562\\
30 686.50382\\
40 1023.7691\\
50 1072.35869\\
100 1668.19503\\
150 1564.48683\\
200 1804.66971\\
250 2173.93693\\
300 1942.55423\\
350 2212.71597\\
400 2228.47637\\
450 2100.54835\\
500 2086.56334\\
550 2139.3581\\
600 2199.40759\\
650 2252.0236\\
700 2174.93242\\
750 2284.40855\\
800 2296.377\\
850 2397.1201\\
900 2444.22461\\
950 2380.22252\\
1000 2462.47889\\
};
\addplot [
color=red,
line width=1.0pt,
only marks,
mark=asterisk,
mark options={scale=0.0},
forget plot
]
plot [error bars/.cd, y dir = both, y explicit]
coordinates{
(10,201.20145) +- (0.0,139.031246177603)(20,407.62562) +- (0.0,185.647000204433)(30,686.50382) +- (0.0,233.063464444597)(40,1023.7691) +- (0.0,268.078977788974)(50,1072.35869) +- (0.0,267.269688786894)(100,1668.19503) +- (0.0,275.042791437383)(150,1564.48683) +- (0.0,273.079596793041)(200,1804.66971) +- (0.0,274.951951754127)(250,2173.93693) +- (0.0,239.18438246924)(300,1942.55423) +- (0.0,259.734751639558)(350,2212.71597) +- (0.0,241.777891440222)(400,2228.47637) +- (0.0,223.417945151487)(450,2100.54835) +- (0.0,245.990357108326)(500,2086.56334) +- (0.0,241.28875074745)(550,2139.3581) +- (0.0,238.593516106546)(600,2199.40759) +- (0.0,219.359588789437)(650,2252.0236) +- (0.0,214.480682120049)(700,2174.93242) +- (0.0,227.710473514836)(750,2284.40855) +- (0.0,209.88666534321)(800,2296.377) +- (0.0,216.719723613883)(850,2397.1201) +- (0.0,196.122321921655)(900,2444.22461) +- (0.0,178.734178048265)(950,2380.22252) +- (0.0,192.340811980663)(1000,2462.47889) +- (0.0,185.244612944156)};

\addplot[solid,fill=red,opacity=2.000000e-01,draw=red,forget plot]
table[row sep=crcr]{
x y\\
10 2633.189 \\
20 3000 \\
30 3000 \\
40 3000 \\
50 3000 \\
100 3000 \\
150 3000 \\
200 3000 \\
250 3000 \\
300 3000 \\
350 3000 \\
400 3000 \\
450 3000 \\
500 3000 \\
550 3000 \\
600 3000 \\
650 3000 \\
700 3000 \\
750 3000 \\
800 3000 \\
850 3000 \\
900 3000 \\
950 3000 \\
1000 3000 \\
1000 226.084 \\
950 141.679 \\
900 286.822 \\
850 204.906 \\
800 141.285 \\
750 114.836 \\
700 77.07 \\
650 153.447 \\
600 102.694 \\
550 83.371 \\
500 64.69 \\
450 85.678 \\
400 64.751 \\
350 37.459 \\
300 21.458 \\
250 29.448 \\
200 21.625 \\
150 16.795 \\
100 5.175 \\
50 6 \\
40 7.17 \\
30 8.037 \\
20 5.301 \\
10 5 \\
};

\addplot [
color=blue,
dash pattern=on 2pt off 2pt on 9pt off 2pt,
line width=2.0pt
]
table[row sep=crcr]{
10 1501.15241\\
20 2970.01618\\
30 2998.58864\\
40 3000\\
50 3000\\
100 3000\\
150 3000\\
200 3000\\
250 3000\\
300 3000\\
350 3000\\
400 3000\\
450 3000\\
500 3000\\
550 3000\\
600 3000\\
650 3000\\
700 3000\\
750 3000\\
800 3000\\
850 3000\\
900 3000\\
950 3000\\
1000 3000\\
};
\addplot [
color=blue,
line width=0.5pt,
only marks,
mark=square,
mark options={scale=0.0},
forget plot
]
plot [error bars/.cd, y dir = both, y explicit]
coordinates{
(10,1501.15241) +- (0.0,240.074136641562)(20,2970.01618) +- (0.0,32.4260371541864)(30,2998.58864) +- (0.0,2.3895590831653)(40,3000) +- (0.0,0)(50,3000) +- (0.0,0)(100,3000) +- (0.0,0)(150,3000) +- (0.0,0)(200,3000) +- (0.0,0)(250,3000) +- (0.0,0)(300,3000) +- (0.0,0)(350,3000) +- (0.0,0)(400,3000) +- (0.0,0)(450,3000) +- (0.0,0)(500,3000) +- (0.0,0)(550,3000) +- (0.0,0)(600,3000) +- (0.0,0)(650,3000) +- (0.0,0)(700,3000) +- (0.0,0)(750,3000) +- (0.0,0)(800,3000) +- (0.0,0)(850,3000) +- (0.0,0)(900,3000) +- (0.0,0)(950,3000) +- (0.0,0)(1000,3000) +- (0.0,0)};

\addplot[solid,fill=blue,opacity=2.000000e-01,draw=blue,forget plot]
table[row sep=crcr]{
x y\\
10 3000 \\
20 3000 \\
30 3000 \\
40 3000 \\
50 3000 \\
100 3000 \\
150 3000 \\
200 3000 \\
250 3000 \\
300 3000 \\
350 3000 \\
400 3000 \\
450 3000 \\
500 3000 \\
550 3000 \\
600 3000 \\
650 3000 \\
700 3000 \\
750 3000 \\
800 3000 \\
850 3000 \\
900 3000 \\
950 3000 \\
1000 3000 \\
1000 3000 \\
950 3000 \\
900 3000 \\
850 3000 \\
800 3000 \\
750 3000 \\
700 3000 \\
650 3000 \\
600 3000 \\
550 3000 \\
500 3000 \\
450 3000 \\
400 3000 \\
350 3000 \\
300 3000 \\
250 3000 \\
200 3000 \\
150 3000 \\
100 3000 \\
50 3000 \\
40 3000 \\
30 3000 \\
20 2927.038 \\
10 44.008 \\
};

\end{semilogxaxis}
\end{tikzpicture}%
      \label{fig:offline}
    }  
    \subfigure[Online results]{
    	 \hspace{-2em}
      \input{figures/Mountain_Online_GP_LBRL.tikz}
      \hfill
      \input{figures/Pendulum_Online_GP_LBRL.tikz}
      \label{fig:online}
    }
    \\
    \ref{leg:offline}
  \end{center}
  \caption{Experimental evaluation. The dashed line shows CTBRL, the dotted line shows LBRL, the solid line shows LSPI, while the dash-dotted line shows GPRL. The error bars denote $95\%$ confidence intervals for the mean (i.e. statistical significance). The shaded regions denote $90\%$ percentile performance (i.e. robustness) across runs. In all cases, CTBRL converges significantly quicker than the other approaches. In addition, as the percentile regions show, it is also much more stable than LBRL, GPRL and LSPI.}
  \label{fig:results}
\end{figure}

\subsection{Domains}
\label{sec:domains}

We consider two well-known continuous state, discrete-action, episodic domains. The first is the inverted pendulum domain and the second is the mountain car domain.

\subsubsection{Inverted pendulum}
The goal in this domain, is to balance a pendulum by applying forces of a mixed magnitude (50 Newtons). The state space consists of two continuous variables, the vertical angle and the angular velocity of the pendulum.
There are three actions: no force, left force or right force. A zero reward is received at each time step except in the case where the pendulum falls.
In this case, a negative (-1) reward is given and a new episode begins.
An episode also ends with $0$ reward after $3000$ steps, after which we consider that the pendulum is successfully balanced.
Each episode starts by setting the pendulum in a perturbed state close to the equilibrium point. More information about the specific dynamics can be found at~\citep{lagoudakis2003least}. The discount factor $\disc$ was $0.95$. The basis we used for LSTD/LSPI, was equidistant $3\times3$ grid of RBFs  over the state space following the suggestions of~\citet{lagoudakis2003least}. This was replicated for each action for the LSTD-$Q$ algorithm used in LSPI.

\subsubsection{Mountain car}
The aim in this domain is to drive an underpowered car to the top of a hill.
Two continuous variables characterise the vehicle state in the domain, its position and its velocity. The objective is to drive an underpowered vehicle up a steep valley from a randomly selected position to the right hilltop (at position $> 0.5$) within $1000$ steps. There are three actions:  forward, reverse and zero throttle. The received reward is $-1$ except in the case where the target is reached (zero reward). At the beginning of each rollout, the vehicle is positioned to a new state, with the position and the velocity uniformly randomly selected. The discount factor is set to $\disc = 0.999$. An equidistant $4 \times 4$ grid of RBFs over the state space plus a constant term is selected for LSTD and LSPI. 

\subsection{Results}
\label{sec:results}
In our results, we show the average performance in terms of number of steps of each method, averaged over 100 runs. For each average, we also plot the 95\% confidence interval for the accuracy of the mean estimate with error bars. In addition, we show the 90\% percentile region of the runs, in order to indicate inter-run variability in performance.

Figure~\ref{fig:offline} shows the results of the experiments in the \emph{offline} case. For the \emph{mountain car}, it is clear that CTBRL is significantly more stable compared to GPRL and LSPI. In contrast to the other two approaches, CTBRL needs only a small number of rollouts in order to discover the optimal policy. For the \emph{pendulum} domain, the performance of both CTBRL and GPRL is almost perfect, as they need only about twenty rollouts in order to discover the optimal policy. On the other hand, LSPI despite the fact that manages to find the optimal policy frequently, around 5\% of its runs fail.

Figure~\ref{fig:online} shows the results of the experiments in the \emph{online} case. For the \emph{mountain car}, CTBRL managed to find an excellent policy in the vast majority of runs, while converging earlier than GPRL and LSPI. Moreover, CTBRL presents a more stable behaviour in contrast to the other two. In the \emph{pendulum} domain, the performance difference  relative to LSPI is even more impressive. It becomes apparent that both CTBRL and GPRL reach near optimal performances with an order of magnitude fewer episodes than LSPI, while the latter remains unstable. In this experiment, we see that CTBRL reaches an optimal policy slightly before GPRL. Although the difference is small, it is very consistent.

The success of CTBRL over the other approaches can be attributed to a number of reasons. Firstly, it could be a better model. Indeed, in the offline results for the mountain car domain, where the starting state distribution is uniform, and all methods have the same data, we can see that CTBRL has a far better performance than everything else. The second could be the more efficient exploration afforded by Thompson sampling. Indeed, in the mountain car online experiments we see that the LBRL performs quite well (Fig.~\ref{fig:online}), even though its offline performance is not very good (Fig.~\ref{fig:offline}). However, Thompson sampling is not sufficient for obtaining a good performance, as seen by both the offline results and the performance in the pendulum domain.

\section{Conclusion}
\label{sec:conclusion}
We proposed a computationally efficient, fully Bayesian approach for the exact inference of unknown dynamics in continuous state spaces. The total computation for inference after $t$ steps is $O(t \ln t)$,  in stark contrast to other non-parametric models such as Gaussian processes, which scale $O(t^3)$. In addition, inference is naturally performed online, with the computational cost at time $t$ being $O(\ln t)$. 

In practice, the computational complexity is orders of magnitude lower for cover trees than GP, even for these problems. We had to use a dictionary and a lot of tuning to make GP methods work, while cover trees worked out of the box. Another disadvantage of GP methods is that it is infeasible to implement Thompson sampling with them. This is because it is not possible to directly sample a function from the GP posterior. Although Thompson sampling confers no advantage in the offline experiments (as the data there were the same for all methods), we still see that the performance of CTBRL is significantly better on average and that it is much more stable.

Experimentally, we showed that cover trees are more efficient both in terms of computation and in terms of reward, relative to GP models that used the same ADP method to optimise the policy and to a linear Bayesian model which used both the same ADP method and the same exploration strategy. We can see that overall the linear model performs significantly worse than both GP-RL and CTBRL, though better than $\epsilon$-greedy LSPI. This shows that the main reason for the success of CTBRL is the cover tree inference and not the linear model itself, or Thompson sampling.

CTBRL is particularly good in online settings, where the exact inference, combined with the efficient exploration provided by Thompson sampling give it an additional advantage. We thus believe that CTBRL is a method that is well-suited for exploration in unknown continuous state problems. 
Unfortunately, it is not possible to implement Thompson sampling in practice using GPs, as there is no reasonable way to sample a function from the GP posterior. 
Nevertheless, we found that in both online and offline experiments (where Thompson sampling should be at a disadvantage) the cover tree method achieved superior performance to Gaussian processes. 

Although we have demonstrated the method in low dimensional problems, higher dimensions are not a problem for the cover tree inference itself. The bottleneck is the value function estimation and ADP. This is independent of the  model used, however. For example, GP methods for estimating the value function ~\citep[c.f.][]{Deisenroth2009} typically have a large number of hyper-parameters for value function estimation, such as choice of representative states and trajectories, kernel parameters and method for updating the dictionary, to avoid problems with many observations.

While in practice ADP  can be performed in the background while inference is taking place, and although we seed the ADP with the previous solution, one would ideally like to use a more incremental approach for that purpose. One interesting idea would be to employ a gradient approach in a similar vein to \citet{Deisenroth2011:pilco}. An alternative approach would be to employ an online method, in order to avoid estimating a policy for the complete space.\footnote{A suggestion made by the anonymous reviewers.} Promising such approaches include running bandit-based tree search methods such as UCT~\citep{ECML:Kocsis+Szepesvari:2006} on the sampled models. 

Another direction of future work is to consider more sophisticated exploration policies, particularly for larger problems. Due to the efficiency of the model, it should be possible to compute near-Bayes-optimal policies by applying the tree search method used by~\citet{DBLP:journals/jair/VenessNHUS11}. Finally, it would be interesting to examine continuous actions. These can be handled efficiently both by the cover tree and the local linear models by making the next state directly dependent on the action through an augmented linear model. While optimising over a continuous action space is challenging, more recent efficient tree search methods such as metric bandits~\citep{Bubeck:X-Armed-Bandits:JMLR} may alleviate that problem.

An interesting theoretical direction would be to obtain regret bounds for the problem. This could perhaps be done building upon the analyses of \citet{Kozat07} for context tree prediction, and of \citet{ortner:continuous-rl} for continuous MDPs. The statistical efficiency of the method could be improved by considering edge-based (rather than node-based) distributions on trees, as was suggested by~\citet{Pereira99anefficient}.

Finally, as the cover tree method only requires specifying an appropriate metric, the method could be applicable to many other problems. This includes both large discrete problems, and partially observable problems. It would be interesting to see if the approach also gives good results in those cases.

\section*{Acknowledgements}
We would like to thank the anonymous reviewers, for their careful and detailed comments and suggestions, for this and previous versions of the paper, which have significantly improved the manuscript. We also want to thank Mikael K\r{a}geb\"{a}ck for additional proofreading. This work was partially supported  by the Marie Curie Project ESDEMUU ``Efficient Sequential Decision Making Under Uncertainty'' , Grant Number 237816 and by an ERASMUS exchange grant.



\bibliography{misc}

\end{document}